\documentclass{article}

\usepackage{times}
\usepackage{graphicx} % more modern
\usepackage{subfigure} 

\usepackage{hyperref}

\usepackage[accepted]{icml2017} %%%% UN-COMMENTED FOR CAMERA-READY

\usepackage{url}
\usepackage{natbib}
\usepackage[shortlabels]{enumitem}
\usepackage{capt-of}
\usepackage[ruled,vlined,linesnumbered]{algorithm2e}
\usepackage{amsmath,amssymb,amsfonts,xfrac}
\usepackage{amsthm}
\usepackage{graphicx}
\usepackage{tabularx}
\newcolumntype{b}{X}
\newcolumntype{s}{>{\hsize=.4\hsize}X}
\usepackage{array}
\usepackage{enumitem}
\newtheorem{theorem}{Theorem}[section]
\newtheorem{corollary}{Corollary}[theorem]

\theoremstyle{definition}

\theoremstyle{remark}

\newtheoremstyle{named}{}{}{\itshape}{}{\bfseries}{.}{.5em}{#1 \thmnote{#3}}
\theoremstyle{named}
\newtheorem*{namedtheorem}{Theorem}
\newtheorem*{namedcorollary}{Corollary}
\usepackage{wrapfig} % Inset figures
\usepackage{xcolor}

\icmltitlerunning{Post-Inference Prior Swapping}

\begin{document} 

\twocolumn[
\icmltitle{Post-Inference Prior Swapping}



% you can specify symbols, otherwise they are numbered in order


\begin{icmlauthorlist}
\icmlauthor{Willie Neiswanger}{mld}
\icmlauthor{Eric Xing}{scs}
\end{icmlauthorlist}

\icmlaffiliation{mld}{Carnegie Mellon University, Machine Learning Department, Pittsburgh, USA}
\icmlaffiliation{scs}{CMU School of Computer Science}

\icmlcorrespondingauthor{Willie Neiswanger}{willie@cs.cmu.edu}

\icmlkeywords{boring formatting information, machine learning, ICML}

\vskip 0.3in
]



\printAffiliationsAndNotice{}  % leave blank if no need to mention equal contribution
%\footnotetext{hi}

%%%%%%%%%%%%%%%%%%%%%%%%%%%%%%%%%%%%%%%%%%%%%%%%%%
%%%%            ABSTRACT 
%%%%%%%%%%%%%%%%%%%%%%%%%%%%%%%%%%%%%%%%%%%%%%%%%%
\begin{abstract} 
While Bayesian methods are praised for their 
ability to incorporate useful prior knowledge, 
in practice, convenient priors that allow for 
computationally cheap or tractable inference
are commonly used.
In this paper, we investigate the following 
question:
\emph{for a given model, is it possible to
compute an inference result with any 
convenient false prior, and afterwards,
given any target prior of interest,
quickly transform this result into the target 
posterior?}
% use any convenient prior to infer a 
% false posterior, and afterwards, 
% given any target prior of interest,
% quickly transform this result
% into the target posterior?}
A potential solution is to use importance sampling (IS).
However, we demonstrate that IS will fail for many
choices of the target prior, depending on 
its parametric form and similarity to 
the false prior.
Instead, we propose prior swapping, a method
that leverages the pre-inferred false posterior
to efficiently generate accurate posterior samples 
under arbitrary target priors.
Prior swapping lets us apply less-costly
inference algorithms to certain models,
and incorporate new or updated prior
information ``post-inference''.
We give theoretical guarantees about our 
method, and demonstrate it empirically
on a number of models and priors.
\end{abstract}

\vspace{-5mm}

%%%%%%%%%%%%%%%%%%%%%%%%%%%%%%%%%%%%%%%%%%%%%%%%%%
%%%%            INTRODUCTION 
%%%%%%%%%%%%%%%%%%%%%%%%%%%%%%%%%%%%%%%%%%%%%%%%%%
\vspace{-2mm}
\section{Introduction}
\label{introduction}
\vspace{-1mm}

There are many cases in Bayesian modeling where
a certain choice of prior distribution allows for
computationally simple or 
tractable inference. For example,
\begin{itemize}[topsep=1pt,itemsep=0pt]
    \item Conjugate priors 
    yield posteriors with a known parametric form and 
    therefore allow for non-iterative, exact
    inference \citep{diaconis1979conjugate}.
    
    \item Certain priors
    yield models with tractable conditional 
    or marginal distributions,
    which allows efficient approximate inference 
    algorithms to be applied (e.g. Gibbs sampling 
    \citep{smith1993bayesian},
    sampling in collapsed models 
    \cite{teh2006collapsed}, 
    or mean-field variational methods 
    \citep{wang2013variational}).
    
    \item Simple parametric priors allow for 
    computationally cheap density queries, 
    maximization, and sampling, 
    which can reduce costs in iterative inference 
    algorithms (e.g. Metropolis-Hastings
    \citep{metropolis1953equation}, 
    gradient-based MCMC
    \citep{neal2011mcmc}, 
    or sequential Monte 
    Carlo \citep{doucet2000sequential}).
\end{itemize}

For these reasons,
one might hope to infer a result under
a convenient-but-unrealistic prior, 
and afterwards, attempt to correct the result.
More generally, given an inference result
(under a convenient prior or otherwise),
one might wish to incorporate updated 
prior information, or see a result under 
different prior assumptions, without
having to re-run a costly inference algorithm.

This leads to the main question of this
paper:
for a given model, is it possible 
to use any convenient false prior to infer a 
false posterior, and afterwards, 
given any target prior of interest,
% quickly transform this result
% into the target posterior?
efficiently and accurately infer
the associated target posterior?

One potential strategy involves sampling
from the false posterior
and reweighting these samples via 
importance sampling (IS). 
However, depending on the chosen
target prior---both its parametric
form and similarity to the false prior---the
resulting inference can be inaccurate 
due to high or infinite variance IS 
estimates
% (which we demonstrate in 
% Sec.~\ref{importanceSampling}).
(demonstrated in 
Sec.~\ref{importanceSampling}).

We instead aim to devise a method that yields 
accurate inferences for arbitrary target priors.
Furthermore, like IS, we want to make use of the 
pre-inferred false posterior,
without simply running standard inference 
algorithms on the target posterior.
Note that most standard inference algorithms are 
iterative and \emph{data-dependent}: 
parameter updates at each iteration involve data, 
and the computational cost or quality 
of each update depends on the amount of data used.
Hence, running inference algorithms directly 
on the target posterior can be costly
(especially given a large amount of data or 
many target priors of interest)
and defeats the purpose of using a convenient false prior.

In this paper, we propose \emph{prior swapping},
an iterative, \emph{data-independent}
method for generating accurate 
posterior samples under arbitrary 
target priors. 
Prior swapping uses the pre-inferred
false posterior to perform
efficient updates that do not depend on 
the data, and thus proceeds very quickly.
We therefore advocate breaking difficult 
inference problems into two easier steps:
first, do inference using the most 
computationally convenient prior for a 
given model, and then, for all future priors 
of interest, use prior swapping.

In the following sections,
we demonstrate the pitfalls of using IS, 
describe the proposed
prior swapping methods for different types
of false posterior inference results (e.g. 
exact or approximate density 
functions, or samples) and
give theoretical guarantees for these
methods. Finally, we show empirical results on 
heavy-tailed and sparsity 
priors in Bayesian generalized linear models,
and relational priors over components in 
mixture and topic models.
%%%%%%%%%%%%%%%%%%%%%%%%%%%%%%%%%%%%%%%%%%%%%%%%%%
%%%%            END OF INTRODUCTION 
%%%%%%%%%%%%%%%%%%%%%%%%%%%%%%%%%%%%%%%%%%%%%%%%%%

%%%%%%%%%%%%%%%%%%%%%%%%%%%%%%%%%%%%%%%%%%%%%%%%%%
%%%%         PRIOR SWAPPING (METHODS)
%%%%%%%%%%%%%%%%%%%%%%%%%%%%%%%%%%%%%%%%%%%%%%%%%%
\vspace{-2mm}
\section{Methodology}
\label{methodology}
\vspace{-1mm}

Suppose we have a dataset of $n$ 
vectors $x^n$ $=$ $\{ x_1,\ldots,
x_n \}$, $x_i \in \mathbb{R}^p$, and
we have chosen a family of models
with the likelihood function
$L(\theta|x^n) = p(x^n|\theta)$,  
parameterized by $\theta \in \mathbb{R}^d$.
Suppose we have a prior distribution over the
space of model parameters $\theta$, with 
probability density function (PDF)
$\pi(\theta)$.
The likelihood and prior define a joint model
with PDF $p(\theta,x^n)$ $=$ $\pi(\theta)
L(\theta|x^n)$.
In Bayesian inference, we are interested in 
computing the posterior (conditional) 
distribution of this joint 
model, with PDF
\begin{align}
    p(\theta|x^n)
    =
    \frac{\pi(\theta) L(\theta|x^n)}
    {\int \pi(\theta) L(\theta|x^n) 
    \hspace{1mm} d\theta}.
\end{align}
Suppose we've chosen a different prior distribution
$\pi_f(\theta)$, which we refer to as a 
\emph{false prior} (while we refer to 
$\pi(\theta)$ as the \emph{target prior}). We
can now define a new posterior
\begin{align}
    p_f(\theta|x^n)
    =
    \frac{\pi_f(\theta) L(\theta|x^n)}
    {\int \pi_f(\theta) L(\theta)|x^n) 
    \hspace{1mm} d\theta} 
\end{align}
which we refer to as a \emph{false posterior}.

We are interested in the following task:
given a false posterior inference result
(i.e. samples from $p_f(\theta|x^n)$, 
or some exact or approximate PDF), 
choose an arbitrary target prior $\pi(\theta)$ and
efficiently sample from the associated target
posterior $p(\theta|x^n)$---or, more generally,
compute an expectation $\mu_h$ $=$ 
$\mathbb{E}_p\left[h(\theta)\right]$ for some 
test function $h(\theta)$ with respect to
the target posterior.

\vspace{-1mm}
\subsection{Importance Sampling and Prior Sensitivity}
\label{importanceSampling}
\vspace{-1mm}
We begin by describing an initial strategy, 
and existing work in a related 
task known as prior sensitivity analysis.

Suppose we have $T$ false posterior samples 
$\{ \tilde{\theta}_t \}_{t=1}^T$
$\sim$ $p_f(\theta|x^n)$.
In importance sampling (IS), 
samples from an importance
distribution are used to estimate the expectation of a 
test function with respect to a target distribution.
A straightforward idea is to use the false posterior as an 
importance distribution, and 
compute the IS estimate
\begin{align}
    \hat{\mu}_h^\text{IS} = 
    \sum_{t=1}^T w(\tilde{\theta}_t) h(\tilde{\theta}_t)
\end{align}
where the weight function
$w(\theta) \propto$ 
$\frac{p(\theta|x^n)}{p_f(\theta|x^n)} \propto$ 
$\frac{\pi(\theta)}{\pi_f(\theta)}$,
% and the weights are normalized, i.e. 
% $\sum_{t=1}^T w(\tilde{\theta}_t) = 1$.
and the T weights are normalized to sum to one.

IS-based methods have been developed
for the task of 
\emph{prior sensitivity analysis} (PSA). In PSA,
the goal is to determine how the posterior varies
over a sequence of priors (e.g. over a parameterized 
family of priors $\pi(\theta;\gamma_i)$, $i=0,1,\ldots$).
Existing work has proposed inferring a single posterior under prior $\pi(\theta;\gamma_0)$, and then 
using IS methods to infer further
posteriors in the sequence \cite{besag1995bayesian,
hastings1970monte,bornn2010efficient}.
% using IS to infer the next posterior
% in the sequence (i.e. under prior $\pi(\theta;\gamma_1)$)
% \cite{XX,XX}, or continuining this for the whole sequence
% with sequential Monte Carlo \cite{XX}.

This strategy is effective when subsequent priors 
 are similar enough, but breaks down
when two priors are sufficiently
dissimilar, or are from ill-matched parametric families, 
which we illustrate in an example below.

Note that, in general for IS,
as $T \rightarrow \infty$,
$\hat{\mu}_h^\text{IS} \rightarrow \mu_h$ 
almost surely.
However, IS estimates can still fail in practice
if $\hat{\mu}_h^\text{IS}$ has high or infinite variance.
If so, the variance of the weights $w(\tilde{\theta}_t)$ 
will be large (a problem often referred to as weight 
degeneracy), which can lead to inaccurate estimates.
In our case, the variance of $\hat{\mu}_h^\text{IS}$ 
is only finite if 
\begin{align}
    \mathbb{E}_{p_f}\left[h(\theta)^2
    \frac{\pi(\theta)^2}{\pi_f(\theta)^2} \right]
    \propto
    \mathbb{E}_p\left[h(\theta)^2
    \frac{\pi(\theta)}{\pi_f(\theta)} \right]
     < \infty.
\end{align}
For a broad class of $h$, this is satisfied if there
exists $M \in \mathbb{R}$ such that 
$\frac{\pi(\theta)}{\pi_f(\theta)} < M, \forall \theta$ 
\cite{geweke1989bayesian}.
Given some pre-inferred $p_f(\theta|x^n)$ with
false prior $\pi_f(\theta)$,
the accuracy of IS thus depends on  
the target prior of interest.
For example, if $\pi(\theta)$ has heavier tails than
$\pi_f(\theta)$, the variance of $\hat{\mu}_h^\text{IS}$ 
will be infinite for many $h$. Intuitively, we 
expect the variance to be higher for $\pi$ that are more
dissimilar to $\pi_f$.

\begin{figure*}[!ht]
        \center{\includegraphics[width=1\textwidth]
        {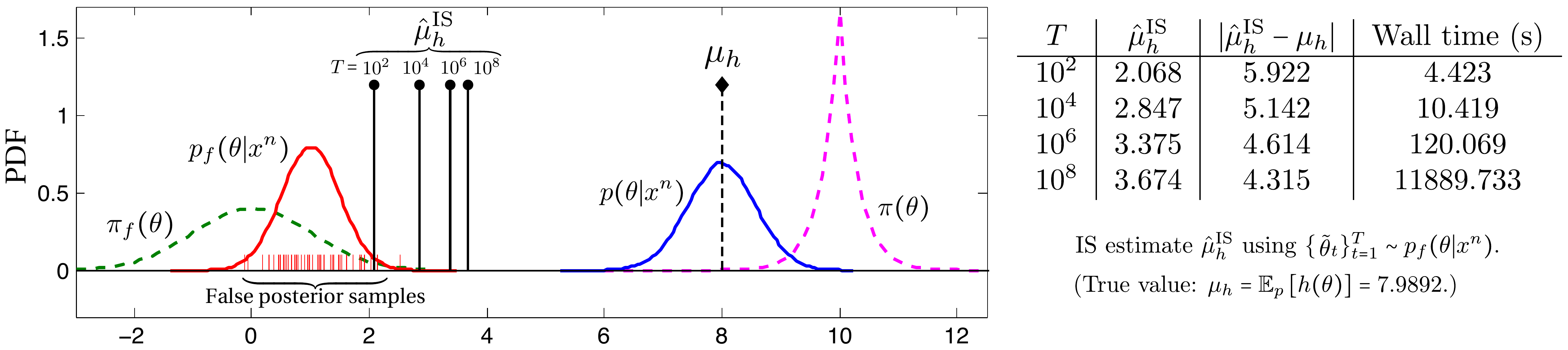}}
        % \vspace{-3mm} \vspace*{-3mm}
        \caption{ \label{fig:runningExample01}
        Importance sampling with false posterior
        samples. As the number of samples $T$ grows, the 
        difference between the IS estimate 
        $\hat{\mu}_h^\text{IS}$ and the true value $\mu_h$
        decreases increasingly slowly. The difference
        remains large even when $T=10^8$.
        See text for analysis.}
        % \vspace{-2mm}
        % \vspace*{-2mm}
\end{figure*}

We show a concrete example of this in 
Fig.~\ref{fig:runningExample01}. Consider a 
normal model for data $x^n \sim \mathcal{N}(\theta,1)$,
% where we're interested in the 
% posterior over $\theta \in \mathbb{R}$.
% Suppose we have a standard normal false prior 
with a standard normal false prior
$\pi_f(\theta) = \mathcal{N}(\theta|0,1)$. 
This yields a closed-form false posterior
(due to the conjugate $\pi_f$), which is also normal.
Suppose we'd like to estimate the posterior expectation under
a Laplace target prior, with mean 10 and 
variance 1, for test function $h(\theta) = \theta$ 
(i.e. an estimate of the target posterior mean).
We draw $T$ false posterior samples 
$\{\tilde{\theta}_t\}_{t=1}^T$
$\sim$ $p_f(\theta|x^n)$, compute weights 
$w(\tilde{\theta}_t)$
and IS estimate $\hat{\mu}_h^\text{IS}$, and
compare it with the true expectation $\mu_h$.

We see in Fig.~\ref{fig:runningExample01} 
that $|\mu_h - \hat{\mu}_h^\text{IS}|$ slows 
significantly as $T$ increases, and maintains a high error
even as $T$ is made very large.
We can analyze this issue theoretically. 
Suppose we want $|\mu_h - \hat{\mu}_h^\text{IS}| < \delta$.
Since we know $p_f(\theta|x^n)$ is normal, we can compute 
a lower bound on the number of 
false posterior samples $T$ that would 
be needed for the expected estimate 
to be within $\delta$ of $\mu_h$.
Namely, if $p_f(\theta|x^n) = \mathcal{N}(\theta|m,s^2)$,
in order for 
$|\mu_h - \mathbb{E}_{p_f}[\hat{\mu}_h^\text{IS}]| < \delta$,
we'd need
\begin{align*}
    T \geq \exp \left\{ \frac{1}{2s^2} 
    (|\mu_h - m| - \delta)^2 \right\}.
\end{align*}
In the example in Fig.~\ref{fig:runningExample01},
we have $m=1$, $s^2 = 0.25$, and $\mu_h = 7.9892$.
Hence, for 
$|\mu_h - \mathbb{E}_{p_f}[\hat{\mu}_h^\text{IS}]| < 1$,
we'd need $T>10^{31}$ samples
(see appendix for full details of this analysis). 
Note that this bound actually has nothing to do with the 
parametric form of $\pi(\theta)$---it is 
based solely on the normal false posterior,
and its distance to the target
posterior mean $\mu_h$. However, even if this distance
was small, the importance estimate
would still have infinite variance due to 
the Laplace target prior.
Further, note that the situation can significantly worsen in
higher dimensions, or if the false posterior has a 
lower variance.

% \begin{align*}
%     \mathbb{E}_{p_f}\left[\min \hspace{1mm} T : 
% |\hat{\mu}_h^\text{IS} - \mu_h| < \delta \right]
%     > O(\exp(T)) [CHECK].
% \end{align*}
% For example, for $|\hat{\mu}_h^\text{IS} - \mu_h| < 0.25$, 
% we'd expect to need \emph{more than} $T = 10^8$ [CHECK] samples
% The situation can significantly worsen in
% higher dimensions, or if the false posterior has a 
% lower variance.

%% TABLE FOR RUNNING EXAMPLE (IS EXAMPLE)
% \begin{center}
% \begin{tabular}{ c|c|c|c }
%  $T$ & $\hat{\mu}_h^\text{IS}$ & $|\hat{\mu}_h^\text{IS}-\mu_h|$ & Wall time (s)\\
%  \hline
%  $10^2$ & 2.068 & 5.922 & 4.423\\
%  $10^4$ & 2.847 & 5.142 & 10.419\\
%  $10^6$ & 3.375 & 4.614 & 120.069\\
%  $10^8$ & 3.674 & 4.315 & 11889.733
% \end{tabular}
% \end{center}

%% TABLE FOR IS RUNNING EXAMPLE (PS EXAMPLE)
% \begin{center}
% \begin{tabular}{ c|c|c|c }
%  $T$ & $\hat{\mu}_h^\text{IS}$ & $|\hat{\mu}_h^\text{IS}-\mu_h|$ & Wall time (s)\\
%  \hline
%  $10^2$ & 7.823 & 0.166 & 0.039\\
%  $10^3$ & 8.021 & 0.032 & 0.381\\
%  $10^4$ & 7.979 & 0.010 & 3.808\\
%  $10^5$ & 7.996 & 0.007 & 38.027
% \end{tabular}
% \end{center}

% Sample estimate $\hat{\mu}_h^\text{PS}$ using $\{\mu_t\}_{t=1}^T \sim p_s(\theta)$:

% True value: $\mu_h = \mathbb{E}_p\left[h(\theta)\right] = 7.9892$

\subsection{Prior Swapping}
\label{priorswapping}
We'd like a method that will work well
even when false and target priors 
$\pi_f(\theta)$ and $\pi(\theta)$ 
are significantly different, or are from 
different parametric families, with performance
that does not worsen (in accuracy nor computational 
complexity) as the priors are made more dissimilar.

Redoing inference for each new target posterior 
can be very costly, especially when the 
data size $n$ is large, because
the per-iteration cost of 
most standard inference algorithms
scales with $n$, and many iterations
may be needed for accurate inference.
This includes both MCMC and sequential monte carlo (SMC)
algorithms (i.e. repeated-IS-methods that infer 
a sequence of distributions). In SMC, 
the per-iteration cost still scales with $n$, 
and the variance estimates can still be infinite 
if subsequent distributions are ill-matched.

% Re-running an inference algorthm (such as MCMC)
% for each new target prior
% can be very time consuming, especially when the data size 
% $n$ is large.
%  We could 
% attempt to perform IS sequentially 
% (e.g. SMC akin to \cite{XXX}), but this too has 
% downsides:  we would need a general
% way to construct a valide sequence of distributions
% from $p_f$ to $p$; 
% the computational cost will be large if 
% $p$ and $p_f$ are far apart;
% intermediate steps in the sequence 
% will involve moving samples, which has 
% cost that depends on the data size $n$;
% and we could still have high or infinite 
% variance estimates if subsequent distributions 
% are ill-matched
% % (which again will depend on the 
% % relation between $\pi$ and $\pi_f$)
% (which will again depend on the form of $\pi$)
% or are too dissimilar.
% Further, we'd like to take advantage of any closed-form 
% or approximate false posterior density, if such 
% an inference result is available, without needing to 
% begin with false posterior samples.

Instead, we aim to leverage the inferred
false posterior to more-efficiently 
compute any future target posterior.
We begin by defining a 
\emph{prior swap density} $p_s(\theta)$. 
Suppose for now that a false posterior inference 
algorithm has returned a density 
function $\tilde{p}_f(\theta)$ (we will give more 
details on $\tilde{p}_f$ later; assume for now
that it is either equal to 
$p_f(\theta|x^n)$ or approximates it). We then
define the prior swap density as
\begin{align}
    p_s(\theta) 
    \propto \frac{\tilde{p}_f(\theta) \pi(\theta)}{\pi_f(\theta)}.
\end{align}

\begin{figure*}[!ht]
        \center{\includegraphics[width=1\textwidth]
        {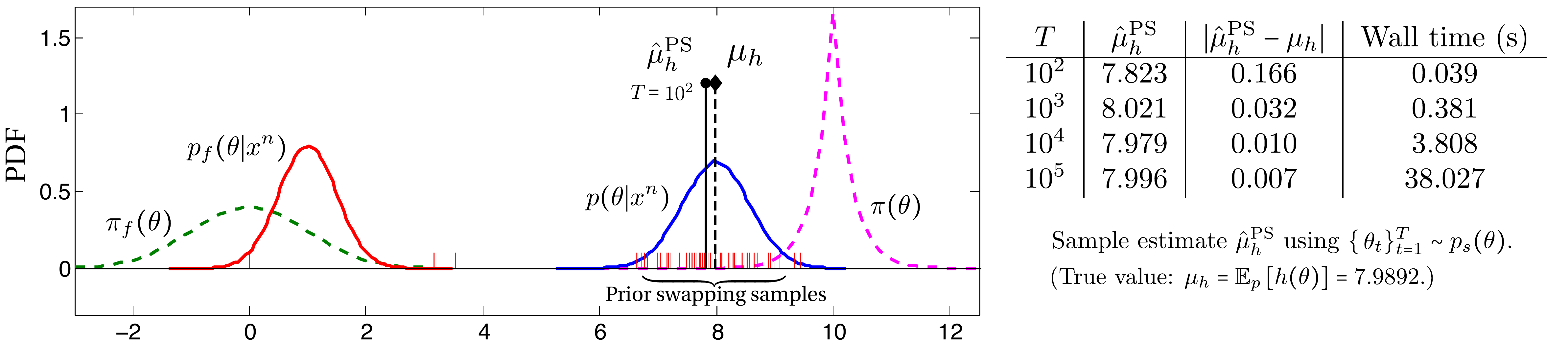}}
        % \vspace{-3mm} \vspace*{-3mm}
        \caption{\label{fig:runningExample02} 
        Using prior swapping to compute estimate
        $\hat{\mu}_h^\text{PS}$ by drawing samples 
        $\{\theta_t\}_{t=1}^T \sim p_s(\theta)$.
        }
        % \vspace{-2mm}
        % \vspace*{-2mm}
\end{figure*}

Note that if $\tilde{p}_f(\theta) = p_f(\theta|x^n)$, then
$p_s(\theta) = p(\theta|x^n)$. However,
 depending on how we represent $\tilde{p}_f(\theta)$, 
$p_s(\theta)$ can have a much simpler analytic 
representation than $p(\theta|x^n)$, which is 
typically defined via a likelihood function 
(i.e. a function of the data) and causes inference
algorithms to have costs that scale 
with the data size $n$.
Specifically, we will only use low-complexity
$\tilde{p}_f(\theta)$ that can 
be evaluated in constant time with 
respect to the data size $n$.

Our general strategy is to use
$p_s(\theta)$ as a surrogate for 
$p(\theta|x^n)$ in standard MCMC 
or optimization procedures, 
to yield data-independent algorithms
with constant cost per iteration.
Intuitively, the likelihood information
is captured by the false posterior---we 
make use of this instead of the
likelihood function, which is costly to evaluate.

More concretely, at each iteration
in standard inference algorithms,
we must evaluate a data-dependent function associated with
the posterior density. For example, we 
evaluate a function proportional to 
$p(\theta|x^n)$ in 
Metropolis-Hastings (MH) 
\citep{metropolis1953equation},
and $\nabla_\theta \log  p(\theta | x^n)$
in gradient-based MCMC methods (such as Langevin dynamics (LD)
\cite{rossky1978brownian} and
Hamiltonian Monte Carlo (HMC) \citep{neal2011mcmc})
and in optimization procedures 
that yield a MAP point estimate.
In prior swapping, we
instead evaluate $p_s(\theta)$ in MH, or 
$\nabla_\theta \log p_s(\theta)$ in 
LD, HMC, or gradient optimization 
to a MAP estimate (see appendix for algorithm pseudocode).
Here, each iteration only requires
evaluating a few simple analytic expressions,
and thus has $O(1)$ complexity with respect to data size.

We demonstrate prior swapping on our
previous example (using a normal false prior
and Laplace target prior) in Fig.~\ref{fig:runningExample02}, 
where we have a closed-form
(normal PDF) $\tilde{p}_f(\theta)$. 
To do prior swapping,
we run a Metropolis-Hastings algorithm
on the target density $p_s(\theta)$. Note
that drawing each sample in this Markov chain 
does not involve the data $x^n$, and can 
be done in constant time with respect to $n$ (which we
can see by viewing the wall time for different $T$).
In Fig.~\ref{fig:runningExample02},
we draw T samples $\{\theta_t\}_{t=1}^T \sim p_s(\theta)$, 
compute a sample estimate 
$\hat{\mu}_h^\text{PS} = \frac{1}{T}\sum_{t=1}^T \theta_t$,
and compare it with the true value $\mu_h$.
We see that $\hat{\mu}_h^\text{PS}$ converges to
$\mu_h$ after a relatively small number of samples T.

\subsection{Prior Swapping with False Posterior Samples}
\label{sampleBasedPS}
The previous method is only applicable if 
our false posterior inference result 
is a PDF $\tilde{p}_f(\theta)$ (such as 
in closed-form inference or variational approximations). 
Here, we develop
prior swapping methods for the setting where 
we only have access to
samples $\{\tilde{\theta}_t\}_{t=1}^{T_f} \sim p_f(\theta|x^n)$.
We propose the following procedure:
\begin{enumerate}[topsep=0pt,itemsep=0pt]
    \item Use $\{ \tilde{\theta}_t \}_{t=1}^{T_f}$ to form
    an estimate $\tilde{p}_f(\theta) \approx p_f(\theta|x^n)$.
    \item  Sample from 
    $p_s(\theta) \propto 
    \frac{\pi(\theta)\tilde{p}_f(\theta)}{\pi_f(\theta)}$ 
    with prior swapping, as before.
\end{enumerate}
Note that, in general, $p_s(\theta)$ only approximates 
$p(\theta|x^n)$. As a final step, after sampling
from $p_s(\theta)$, we can:
\begin{enumerate}[topsep=0pt,itemsep=0pt,start=3]
    \item Apply a correction to samples from $p_s(\theta)$.
\end{enumerate}

We will describe two methods for applying a correction to 
$p_s$ samples---one involving importance sampling, and one 
involving semiparametric density estimation.
Additionally, we will discuss forms for $\tilde{p}_f(\theta)$,
guarantees about these forms, and how to optimize the 
choice of $\tilde{p}_f(\theta)$.
In particular, we will argue why (in constrast to the
initial IS strategy) these methods do not fail when 
$p(\theta|x^n)$ and $p_f(\theta|x^n)$ are
very dissimilar or have ill-matching parametric forms.

\paragraph{Prior swap importance sampling.}
Our first proposal for applying a correction 
to prior swap samples involves IS:
after estimating some $\tilde{p}_f(\theta)$, and
 sampling $\{\theta_t\}_{t=1}^T \sim p_s(\theta)$,
we can treat $\{\theta_t\}_{t=1}^T$ as importance samples, 
and compute the IS estimate
\begin{align}
    \hat{\mu}_h^\text{PSis} = 
    \sum_{t=1}^T w(\theta_t) h(\theta_t)
\end{align}
where the weight function is now
\begin{align}
    \label{eq:newWeightFunction}
    w(\theta) \propto \frac{p(\theta|x^n)}{p_s(\theta)}
    \propto \frac{p_f(\theta|x^n)}{\tilde{p}_f(\theta)}
\end{align}
and the weights are normalized so that 
$\sum_{t=1}^T w(\theta_t) = 1$.

The key difference between this and the previous 
IS strategy is the weight function.
Recall that, previously, an accurate estimate depended
on the similarity between $\pi(\theta)$ and $\pi_f(\theta)$;
both the distance to and parametric form of $\pi(\theta)$ 
could produce high or infinite variance estimates.
This was an issue because we wanted the 
procedure to work well for any $\pi(\theta)$.
Now, however, the performance depends on the
similarity between $\tilde{p}_f(\theta)$ and 
$p_f(\theta|x^n)$---and by using the false posterior
samples, we can estimate a $\tilde{p}_f(\theta)$
that well approximates $p_f(\theta|x^n)$.
Additionally, we can prove that certain
choices of $\tilde{p}_f(\theta)$
\emph{guarantee} a finite variance IS estimate.
Note that the variance of $\hat{\mu}_h^\text{PSis}$
is only finite if 
\begin{align*}
    \mathbb{E}_{p_f}\left[h(\theta)^2
    \frac{p_f(\theta|x^n)^2}{\tilde{p}_f(\theta)^2} \right]
    \propto
    \mathbb{E}_p\left[h(\theta)^2
    \frac{p_f(\theta|x^n)}{\tilde{p}_f(\theta)} \right]
     < \infty.
\end{align*}
To bound this, it is sufficient to show that 
there exists $M \in \mathbb{R}$ such that 
$\frac{p_f(\theta|x^n)}{\tilde{p}_f(\theta)} < M$
for all $\theta$ (assuming a test
function $h(\theta)$ with finite variance) 
\cite{geweke1989bayesian}.
% While any parametric approximation could be 
% used in prior swapping, some might 
% yield infinite-variance estimates.
To satisfy this condition, we will propose
a certain parametric family $\tilde{p}_f^\alpha(\theta)$.
% (and later we will show how to learn the optimal $\alpha$).
Note that, to maintain a prior swapping procedure with
$O(1)$ cost, we want a $\tilde{p}_f^\alpha(\theta)$ that can
be evaluated in constant time. In general, a 
$\tilde{p}_f^\alpha(\theta)$ with fewer terms will 
yield a faster procedure. With these in mind,
we propose the following family of densities.

\textbf{Definition.} For a parameter 
$\alpha = (\alpha_1,\ldots,\alpha_k)$, 
$\alpha_j \in \mathbb{R}^p$,
$k>0$, let density $\tilde{p}_f^\alpha(\theta)$ satisfy
\begin{align}
    \label{eq:paraFalsePost}
    \tilde{p}_f^\alpha(\theta)
    \propto
    \pi_f(\theta)
    \prod_{j=1}^k p(\alpha_j|\theta)^{n/k}
\end{align}
where $p(\alpha_j|\theta)$ denotes the model 
conditional PDF.

The number of terms in $\tilde{p}_f^\alpha(\theta)$ (and 
cost to evaluate) is determined by the parameter $k$.
Note that this family is inspired by the true form 
of the false posterior $p_f(\theta|x^n)$.
However, $\tilde{p}_f^\alpha(\theta)$ has constant-time
evaluation, and we can estimate its parameter $\alpha$
using samples 
$\{\tilde{\theta}_t\}_{t=1}^{T_f} \sim p_f(\theta|x^n)$.
Furthermore, we have the following guarantees.

% \begin{theorem}
% For any $\alpha = (\alpha_1,\ldots,\alpha_k) \subset \mathbb{R}^p$
% and $k>0$, there exists a density $\tilde{p}_f^\alpha(\theta)$ 
% that satisfies Eq.~(\ref{eq:paraFalsePost}).
% \end{theorem}

\begin{theorem}
For any $\alpha = (\alpha_1,\ldots,\alpha_k) \subset \mathbb{R}^p$
and $k>0$ let 
$\tilde{p}_f^\alpha(\theta)$ be defined 
as in Eq.~(\ref{eq:paraFalsePost}).
Then, there exists $M>0$ such that 
$\frac{p_f(\theta|x^n)}{\tilde{p}_f^\alpha(\theta)} < M$, for all $\theta \in \mathbb{R}^d$.
\end{theorem}

\begin{corollary}
For $\{\theta_t\}_{t=1}^T \sim p_s^\alpha(\theta) \propto 
\frac{\tilde{p}_f^\alpha(\theta) \pi(\theta)}{\pi_f(\theta)}$,
$w(\theta_t) = 
\frac{p_f(\theta_t|x^n)}{\tilde{p}_f^\alpha(\theta_t)}
\left(\sum_{r=1}^T
\frac{p_f(\theta_r|x^n)}{\tilde{p}_f^\alpha(\theta_r)}
\right)^{-1}$, and 
test function that satisfies $\text{Var}_p\left[h(\theta)\right] < \infty$,
the variance of IS estimate $\hat{\mu}_h^\text{PSis} = 
\sum_{t=1}^T h(\theta_t) w(\theta_t)$
is finite.
\end{corollary}
Proofs for these theorems are given in the appendix.

Note that we do not know the normalization constant
for $\tilde{p}_f^\alpha(\theta)$. This 
is not an issue for its use in prior swapping, since
we only need access to a function proportional
to $p_s^\alpha(\theta) \propto 
\tilde{p}_f^\alpha(\theta) \pi(\theta) \pi_f(\theta)^{-1}$
in most MCMC algorithms.
However, we still need to estimate $\alpha$, which \emph{is} 
an issue because the unknown normalization constant
is a function of $\alpha$.
Fortunately, we can use the method of \emph{score matching}
\cite{hyvarinen2005estimation} 
to estimate $\alpha$ given a density such 
as $\tilde{p}_f^\alpha(\theta)$ with unknown 
normalization constant.

Once we have found an optimal parameter $\alpha^*$, 
we draw samples from $p_s^{\alpha^*}(\theta)$ $\propto$ $\tilde{p}_f^{\alpha^*}(\theta)
\pi(\theta) \pi_f(\theta)^{-1}$, compute
weights for these samples (Eq.~(\ref{eq:newWeightFunction})), 
and compute the IS estimate $\hat{\mu}_h^\text{PSis}$.
We give pseudocode for the full prior swap importance 
sampling procedure in Alg.~\ref{alg:priorSwapIS}.

\begin{algorithm}[h]
\caption{Prior Swap Importance Sampling}
\label{alg:priorSwapIS}
    \KwIn{False posterior samples 
    $\{ \tilde{\theta}_t \}_{t=1}^{T_f} \sim 
    p_f(\theta | x^n)$.}
    \KwOut{IS estimate $\hat{\mu}_h^\text{PSis}$.}
    \vspace{2pt}
    Score matching:
    estimate $\alpha^*$ using 
    $\{ \tilde{\theta}_t \}_{t=1}^{T_f}$.\\
    Prior swapping: sample
    $\{\theta_t\}_{t=1}^T \sim p_s^{\alpha^*}(\theta) \propto 
    \frac{\tilde{p}_f^{\alpha^*}(\theta)\pi(\theta)}{\pi_f(\theta)}$.\\
    Importance sampling: compute $\hat{\mu}_h^\text{PSis}
    = \sum_{t=1}^T h(\theta_t)w(\theta_t)$.
    % IS: compute $\hat{\mu}_h^\text{PSis}
    % = \frac{1}{T} \sum_{t=1}^T h(\theta_t)w(\theta_t)$.
    %
    %
    % , where
    % $w(\theta_t) = 
    % \frac{p_f(\theta_t|x^n)}{\tilde{p}_f^\alpha(\theta_t)}
    % \left(\sum_{r=1}^T 
    % \frac{p_f(\theta_r|x^n)}{\tilde{p}_f^\alpha(\theta_r)}
    % \right)^{-1}$.
\end{algorithm}

\paragraph{Semiparametric prior swapping.}
In the previous method, we chose a parametric form for 
$\tilde{p}_f^\alpha(\theta)$; in general, even 
the optimal $\alpha$ will yield
an inexact approximation to $p_f(\theta|x^n)$.
Here, we aim to incorporate methods that
return an increasingly exact estimate 
$\tilde{p}_f(\theta)$ when given more 
false posterior samples $\{\tilde{\theta}_t\}_{t=1}^{T_f}$.

One idea is to use a nonparametric kernel
density estimate $\tilde{p}_f^{np}(\theta)$
and plug this into $p_s^{np}(\theta) \propto
\tilde{p}_f^{np}(\theta) \pi(\theta) \pi_f(\theta)^{-1}$.
However, nonparametric density estimates can yield 
inaccurate density tails and fare badly 
in high dimensions. To help mitigate these problems, 
we turn to a semiparametric estimate, which begins 
with a parametric estimate, and adjusts it as 
samples are generated. 
In particular, we use a density estimate
that can be viewed as the product of a parametric density estimate
and a nonparametric correction function \cite{hjort1995nonparametric}. 
This density estimate is consistent as the 
number of samples $T_f \rightarrow \infty$.
Instead of (or in addition to) correcting 
prior swap samples with importance sampling,
we can correct them by updating the 
nonparametric correction function as we
continue to generate false posterior samples. 

Given $T_f$ samples $\{\tilde{\theta}_t\}_{t=1}^{T_f}$ 
$\sim$ $p_f(\theta|x^n)$, we write the semiparametric
false posterior estimate as
\begin{align}
    \label{eq:semiparFP}
    \tilde{p}_f^{sp}(\theta) 
    = \frac{1}{T_f} \sum_{t=1}^{T_f}
    \left[ \frac{1}{b^d}
    K \left(\frac{\|\theta - \widetilde{\theta}_t\|}{b}\right)
    \frac{\tilde{p}_f^\alpha(\theta)}
        {\tilde{p}_f^\alpha(\tilde{\theta}_t)}
    \right] ,
\end{align}
where $K$ denotes a probability density kernel,
with bandwidth $b$, where $b \rightarrow 0$
as $T_f \rightarrow \infty$ (see \citep{wasserman2006all} 
for details on probability density kernels 
and bandwidth selection). 
The semiparametric prior swap density is then
\begin{align}
    \label{eq:semiparEst}
    p_s^{sp}(\theta) 
    \propto
    &\frac{\tilde{p}_f^{sp}(\theta) \pi(\theta)}
        {\pi_f(\theta)}
    =
    \frac{1}{T_f} \sum_{t=1}^{T_f}
    \frac{
    K \left(\frac{\|\theta - \tilde{\theta}_t\|}{b}\right) 
    \tilde{p}_f^\alpha(\theta)  \pi(\theta) }
    { \tilde{p}_f^\alpha(\tilde{\theta}_t) \pi_f(\theta) b^d} 
    \nonumber \\
    % =
    % & \left[ \frac{\hat{f}_{\phi|x}^p (\theta)  f_\theta(\theta)}
    % {f_\phi(\theta)} \right] 
    % \left[ \frac{1}{T} \sum_{t=1}^T 
    % \frac{K \left(\frac{\|\theta - \widetilde{\theta}_t\|}{h}\right)}
    % {\hat{f}_{\phi|x}^p(\widetilde{\theta}_t) h^d}
    % \right] \nonumber \\
    \propto
    & \left[ p_s^\alpha(\theta) \right]
    \left[   \frac{1}{T_f} \sum_{t=1}^{T_f}
    \frac{K \left(\frac{\|\theta - \tilde{\theta}_t\|}{b}\right)}
    {\tilde{p}_f^\alpha(\tilde{\theta}_t)}
    \right].
    % \left[ \frac{1}{T} \sum_{t=1}^T w_t
    % K \left(\frac{\|\theta - \widetilde{\theta}_t\|}{h}\right)
    % \right]. \nonumber
\end{align}
Hence, the prior swap density $p_s^{sp}(\theta)$ is 
proportional to the product of two densities: 
the parametric prior swap density $p_s^\alpha(\theta)$, 
and a correction density.
To estimate expectations with 
respect to $p_s^{sp}(\theta)$, we can 
% sample $\{\theta_t\}_{t=1}^T$ $\sim$ $p_s^\alpha(\theta)$ 
follow Alg.~\ref{alg:priorSwapIS} as before,
but replace the weight function in the final IS estimate with
\begin{align}
    w(\theta) \propto \frac{p_s^{sp}(\theta)}{p_s^\alpha(\theta)}
    \propto
    \frac{1}{T_f} \sum_{t=1}^{T_f} 
    \frac{K \left(\frac{\|\theta - \tilde{\theta}_t\|}{b}\right)}
    {\tilde{p}_f^\alpha(\tilde{\theta}_t)}.
\end{align}

One advantage of this strategy is that 
computing the weights doesn't require the data---it 
thus has constant cost with respect to data size $n$
(though its cost does increase with the 
number of false posterior samples $T_f$).
Additionally, as in importance sampling,
we can prove that this procedure 
yields an exact estimate of $\mathbb{E}[h(\theta)]$, 
asymptotically, as $T_f \rightarrow \infty$ 
(and we can provide an explicit bound on
the rate at which $p_s^{sp}(\theta)$ converges to $p(\theta|x^n)$).
We do this by showing that
$p_s^{sp}(\theta)$ is consistent for $p(\theta|x^n)$.

\begin{theorem}
Given false posterior samples 
$\{\tilde{\theta}_t\}_{t=1}^{T_f} \sim p_f(\theta|x^n)$
and $b \asymp T_f^{-1/(4+d)}$, the estimator
$p_s^{sp}$ is 
consistent for $p(\theta|x^n)$, 
i.e. its mean-squared error satisfies
\begin{align}
    \sup_{ p(\theta|x^n) } \hspace{2mm}
    \mathbb{E}\left[ \int\left(
            p_s^{sp}(\theta) - 
            p(\theta|x^n)\right)^2 
            d\theta \right]
    < \frac{c}{T_f^{4/(4+d)}} \nonumber
\end{align}
for some $c>0$ and $0<b\leq 1$.
\end{theorem}

% \begin{theorem}
% The semiparametric prior swapping procedure  
% generates samples from $p_s^{sp}$  
% (Eq.~\ref{eq:semiparEst})
% as $T \rightarrow \infty$.
% \end{theorem}

\noindent The proof for this theorem is given in the appendix.

%%%%%%%%%%%%%%%%%%%%%%%%%%%%%%%%%%%%%%%%%%%%%%%%%%
%%%%       END OF PRIOR SWAPPING (METHODS)
%%%%%%%%%%%%%%%%%%%%%%%%%%%%%%%%%%%%%%%%%%%%%%%%%%

%%%%%%%%%%%%%%%%%%%%%%%%%%%%%%%%%%%%%%%%%%%%%%%%%%
%%%%             EMPIRICAL RESULTS 
%%%%%%%%%%%%%%%%%%%%%%%%%%%%%%%%%%%%%%%%%%%%%%%%%%
\section{Empirical Results}
\label{empiricalResults}
We show empirical results on
Bayesian generalized linear models (including linear 
and logistic regression) with sparsity 
and heavy tailed priors, and on latent factor models (including 
mixture models and topic models) with relational priors over factors
(e.g. diversity-encouraging, agglomerate-encouraging, etc.). 
We aim to demonstrate empirically that prior swapping 
efficiently yields correct samples and, in some cases,
allows us to apply certain inference algorithms to 
more-complex models than was previously possible.
In the following experiments, we will refer to the following 
procedures:
\begin{itemize}[topsep=1mm,itemsep=0mm]
    \item \textbf{Target posterior inference:} 
    some standard inference algorithm (e.g. MCMC) 
    run on $p(\theta|x^n)$.
    
    \item \textbf{False posterior inference:} 
    some standard inference algorithm run on 
    $p_f(\theta|x^n)$.
    
    \item \textbf{False posterior IS:} 
    IS using samples from $p_f(\theta|x^n)$.
    % to compute $\hat{\mu}_h^\text{IS}$.
    
    \item \textbf{Prior swap exact:} 
    prior swapping with closed-form 
    $\tilde{p}_f(\theta) = p_f(\theta|x^n)$.
    
    \item \textbf{Prior swap parametric:} 
    prior swapping with parametric 
    $\tilde{p}_f^\alpha(\theta)$ given 
    by Eq.~(\ref{eq:paraFalsePost}).
    
    \item \textbf{Prior swap IS:} 
    correcting samples from 
    $\tilde{p}_f^\alpha(\theta)$ with IS.
    
    \item \textbf{Prior swap semiparametric:} 
    correcting samples from 
    $\tilde{p}_f^\alpha(\theta)$ with the
    semiparametric estimate IS procedure.
    
\end{itemize}

To assess performance, we choose a 
test function $h(\theta)$, and compute
the Euclidean distance between
$\mu_h = \mathbb{E}_p[h(\theta)]$
and some estimate $\hat{\mu}_h$ returned by 
a procedure. We denote this performance 
metric by
\emph{posterior error} $= 
\|\mu_h - \hat{\mu}_h\|_2$.
Since $\mu_h$ is typically not available
analytically, we run a single chain 
of MCMC on the target posterior for one million
steps, and use these samples as ground truth
to compute $\mu_h$.
For timing plots, to assess error of a method
at a given time point, we collect samples
drawn before this time point, remove the first quarter
as burn in, and add the time it takes to compute 
any of the corrections.

\subsection{Sparsity Inducing and Heavy Tailed Priors 
in Bayesian Generalized Linear Models}
Sparsity-encouraging regularizers
have gained a high level of popularity over the past
decade due to their ability to produce 
models with greater interpretability and parsimony. For example,
the $L_1$ norm has been used to induce sparsity 
with great effect \citep{tibshirani1996regression}, and has been shown 
to be equivalent to a mean-zero independent Laplace prior 
\citep{tibshirani1996regression,seeger2008bayesian}.
In a Bayesian setting, inference given 
a sparsity prior can be difficult, and often 
requires a computationally intensive method (such as
MH or HMC) or posterior approximations 
(e.g. expectation propagation
\citep{minka2001expectation})
that make factorization or parametric assumptions
\citep{seeger2008bayesian,gerwinn2010bayesian}.
We propose a cheap yet accurate solution:
first get an inference result with a more-tractable prior (such
as a normal prior), and then use prior swapping
to quickly convert the result to the posterior given a 
sparsity prior.

Our first set of experiments are on 
Bayesian linear regression models, 
which we can write as
$
    y_i = X_i\theta + \epsilon, \hspace{1mm}
    \epsilon \sim \mathcal{N}(0,\sigma^2), \hspace{1mm}
    \theta \sim \pi, \hspace{1mm}
    i=1\text{,...,}n
$.
For $\pi$, we compute results on
Laplace, Student's t, and VerySparse (with PDF 
$\text{VerySparse}(\sigma)$ $=$ 
$\prod_{i=1}^d  \frac{1}{2\sigma} \exp\{-|\theta_i|^{0.4} / \sigma \}$
\cite{seeger2008bayesian}) priors.
Here, a normal $\pi_f$ is conjugate and allows for 
exact false posterior inference.
Our second set of experiments are on 
Bayesian logistic regression models, which we write as 
$   y_i \sim \text{Bern}(p_i), \hspace{1mm}
    p_i = \text{logistic}(X_i\theta), \hspace{1mm}
    \theta \sim \pi, \hspace{1mm}
    i=1\text{,...,}n$.
which we will pair with both heavy tailed priors
and a hierarchical target prior
$ \pi = \mathcal{N}(0,\alpha^{-1}I), \hspace{1mm}
\alpha \sim \text{Gamma}(\gamma,1)$. 
For these experiments, we also use a normal $\pi_f$.
However, this false prior is no longer conjugate, and 
so we use MCMC to sample from $p_f(\theta|x^n).$

For linear regression, we use the 
YearPredictionMSD data set\footnote{\url{https://archive.ics.uci.edu/ml/datasets/YearPredictionMSD}}, 
($n=515345$, $d=90$), in which 
regression is used to predict the year
associated with a a song,
and for logistic regression we use 
the MiniBooNE particle identification data
set\footnote{\url{https://archive.ics.uci.edu/ml/datasets/MiniBooNE+particle+identification}},
($n=130065$, $d=50$),
in which binary classification is used to
distinguish particles.

In Fig.~\ref{fig:linRegLogReg}, we compare prior swapping 
and IS methods, in order to show that 
the prior swapping procedures yield accurate posterior 
estimates, and to compare their speeds of convergence.
We plot posterior error vs. wall time for each method's 
estimate of the posterior mean 
$\mathbb{E}_p[h(\theta)] = \mathbb{E}_p[\theta]$
for two sparsity target priors (Laplace and VerySparse),
for both linear and logistic regression.
In linear regression (only), since the normal 
conjugate $\pi_f$ allows us to compute a closed 
form $p_f(\theta|x^n)$, we can
run the \emph{prior swap exact} method,
where $\tilde{p}_f(\theta) = p_f(\theta|x^n)$.
However, we can also sample from 
$p_f(\theta|x^n)$ to compute 
$\tilde{p}_f^{\alpha^*}(\theta)$, 
and therefore compare methods such as
\emph{prior swap parametric} and
the two correction methods.
In logistic regression, we do not have a 
closed form $p_f(\theta|x^n)$; here, we
only compare the methods that make use of
samples from $p_f(\theta|x^n)$.
In Fig.~\ref{fig:linRegLogReg},
we see that the prior swapping methods 
(particularly prior swap IS) quickly
converge to nearly zero posterior error.
Additionally, in linear regression,
we see that prior swap parametric, 
using $\tilde{p}_f(\theta) = \tilde{p}_f^{\alpha^*}(\theta)$,
yields similar posterior error as prior swap exact, which 
uses $\tilde{p}_f(\theta) = p(\theta|x^n)$.

\begin{figure*}[!ht]
        \makebox[\textwidth][c]{
        \hspace{-5mm}
        \includegraphics[width=1\textwidth]{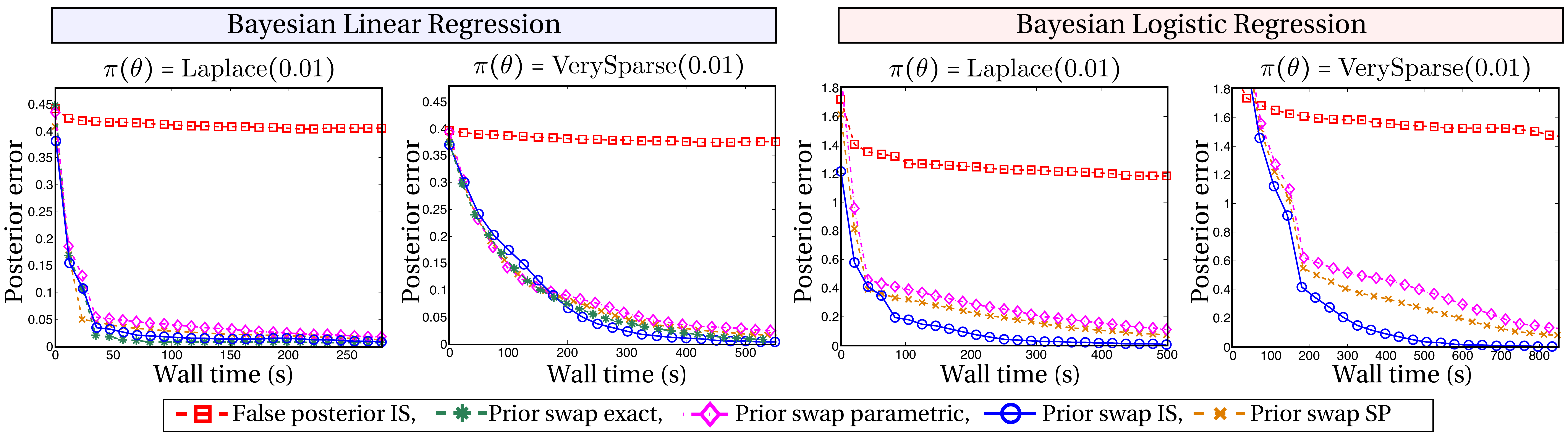}}
        \vspace{-3mm} \vspace*{-3mm}
        \caption{\label{fig:linRegLogReg} 
        Comparison of prior swapping and IS methods
        for Bayesian linear and logistic regression
        under Laplace and VerySparse target priors.
        The prior swapping methods (particularly prior swap exact 
        and prior swap IS) quickly converge to low posterior errors.
        }
        % \vspace{-1mm} \vspace*{-1mm}
\end{figure*}

\begin{figure*}[!ht]
        \makebox[\textwidth][c]{
        % \hspace{-2mm} 
        \includegraphics[width=1\textwidth]{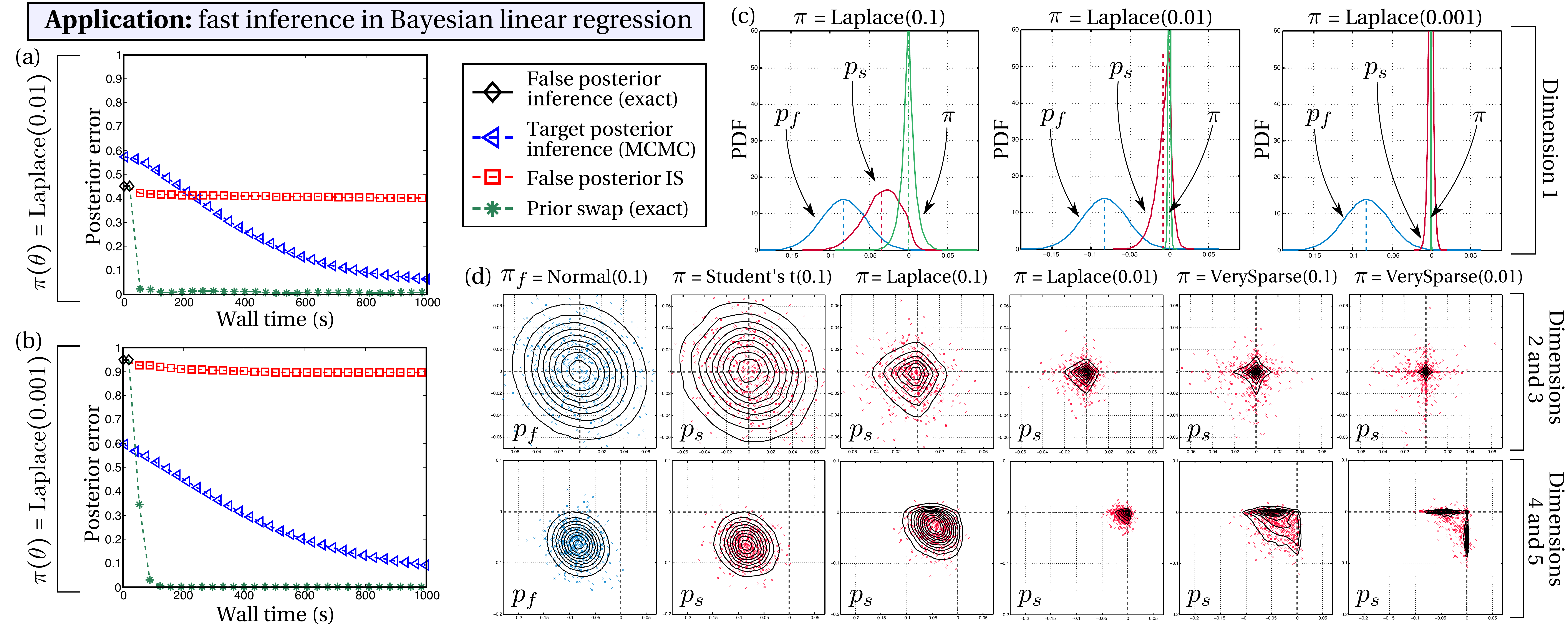}}
        \vspace{-2mm} \vspace*{-2mm}
        \caption{\label{fig:lmApp} 
        Prior swapping for fast inference in Bayesian linear
        models with sparsity and heavy-tailed priors:
        (a-b) Convergence plots showing that prior swapping 
        performs accurate inference faster than the comparison
        methods and is robust to changing $\pi$.
        (c) Inferred 1-d density marginals when prior 
        sparsity is increased. (d) Prior swapping results 
        for a variety of different sparsity priors.
        }
        \vspace{-2mm} \vspace*{-2mm}
\end{figure*}

In Fig.~\ref{fig:lmApp}, we show how prior 
swapping can be used for fast inference
in Bayesian linear models with sparsity
or heavy-tailed priors. 
We plot the time needed to first compute the 
false posterior (via exact inference) and then run 
prior swapping (via the MH procedure) on some 
target posterior, and compare this with the MH 
algorithm run directly on the target posterior.
In (a) and (b) we show convergence plots 
and see that prior swapping performs faster inference
(by a few orders of magnitude) than direct MH.
In plot (b) we reduce the variance of the target
prior; while this hurts the accuracy 
of false posterior IS, prior swapping
still quickly converges to zero error.
In (c) we show 1-d density marginals as we 
increase the prior sparsity, and in (d) we show
prior swapping results for various sparsity priors.

In the appendix, we also include results on
logistic regression with the hierarchical 
target prior, as well as results for synthetic 
data where we are able to compare timing and 
posterior error as we tune $n$ and $d$.

% In Fig.~\ref{fig:lrAll}, we show results for hierarchical
% logistic regression. In (a) and (b) we vary the number of 
% observations ($n$=10-120,000) and see that prior 
% swapping has a constant wall time while
% the wall times of both MCMC and VI increase with $n$. 
% In (b) we see that the prior 
% swapping methods achieve the same test error as the standard
% inference methods. In (c) and (d) we vary the number
% of dimensions ($d$=1-40). In this case, all methods have 
% increasing wall time, and again the test errors match.
% In (e), (f), and (g), we vary the prior hyperparameter 
% ($\gamma$=1-1.05). For prior swapping,
% we infer a single $f_{\phi|x}$ (using $\gamma = 1.025$)
% with both MCMC and VI, and compute \emph{all other}
% hyperparameter results using this $f_{\phi|x}$.
% This demonstrates that prior swapping can quickly infer
% correct results over a range of hyperparameters.
% Here, the asymptotically-exact prior swapping 
% method matches the test error
% of MCMC slightly better than the parametric method.

\begin{figure*}[!ht]
        \makebox[\textwidth][c]{\includegraphics[width=1\textwidth]{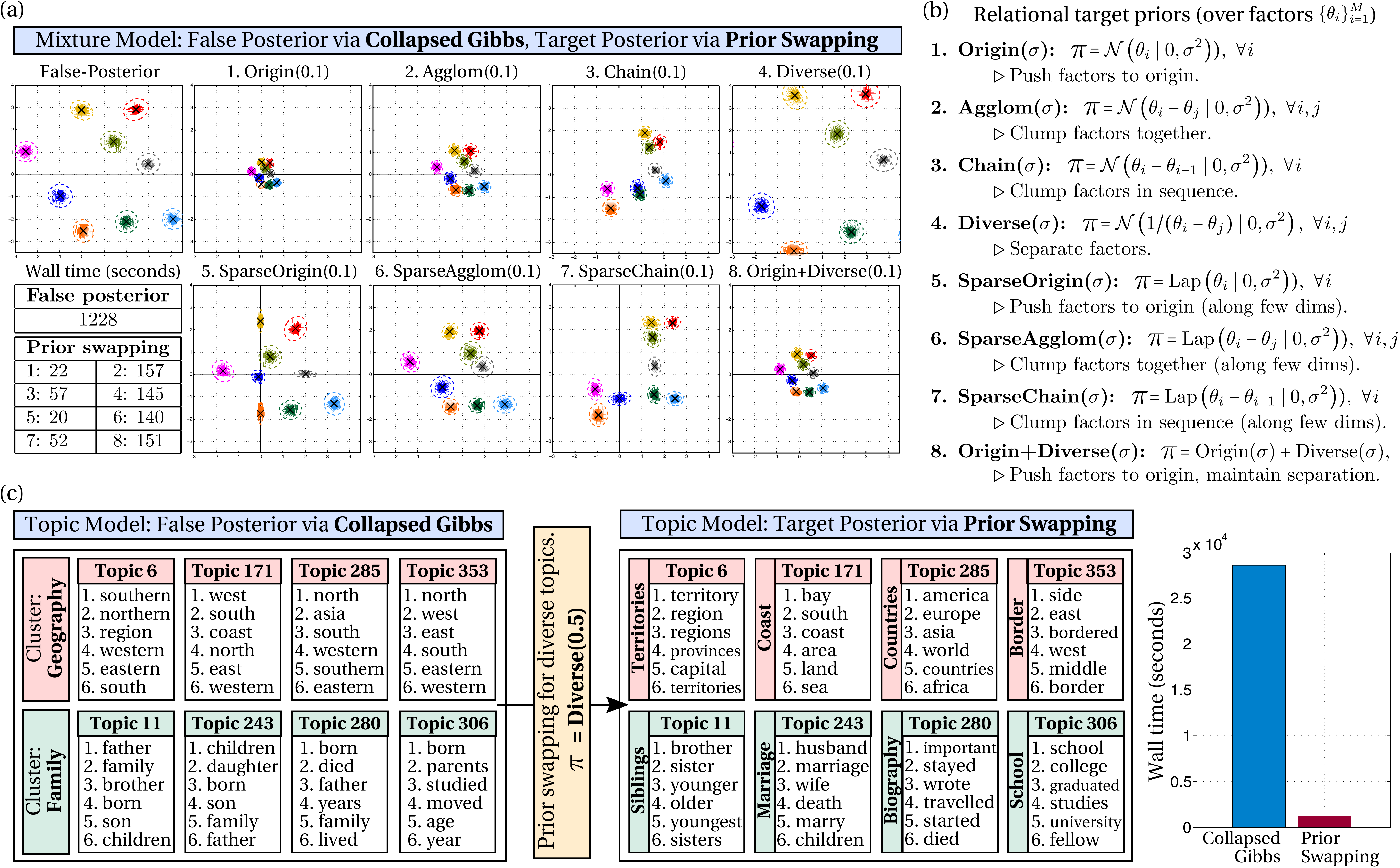}}
        \vspace{-2mm} \vspace*{-2mm}
        \caption{\label{fig:latentFactor}
        Latent factor models: (a) Prior swapping results for 
        relational target priors
        (defined in (b)) over components in a mixture model. (c)
        Prior swapping with a diversity-promoting target prior
        on an LDA topic model (Simple English Wikipedia corpus) 
        to separate redundant topic clusters; the top 6 words per topic
        are shown. In (a, c) we show 
        wall times for the initial inference and prior swapping.
        } 
        \vspace{-1mm} \vspace*{-1mm}
\end{figure*}
\subsection{Priors over Factors in Latent Variable Models}
Many latent variable models in machine learning---such 
as mixture models, topic models, probabilistic matrix 
factorization, and others---involve a set of latent factors 
(e.g. components or topics).
Often, we'd like to use priors that encourage 
interesting behaviors among the 
factors. For example,
we might want dissimilar factors through a 
diversity-promoting prior 
\cite{kwok2012priors,xie2016diversity}
or for the factors to show some sort of 
sparsity pattern \cite{mayrink2013sparse,
knowles2011nonparametric}. 
Inference in such models is often computationally 
expensive or designed on a case-by-case basis 
\cite{xie2016diversity,knowles2011nonparametric}.

However, when conjugate priors are placed
over the factor parameters, collapsed
Gibbs sampling can be applied.
In this method,
the factor parameters are integrated out, leaving 
only a subset of variables; on these, the conditional 
distributions can be computed analytically, which allows for Gibbs 
sampling over these variables.
Afterwards, samples of the collapsed 
factor parameters can be computed. 
Hence, we propose the following strategy: first, assign a prior for the 
factor parameters that allows for collapsed Gibbs sampling;
afterwards, reconstruct the factor samples and 
apply prior swapping for more complex relational priors over the 
factors. We can thus perform convenient inference in the 
collapsed model, yet apply more-sophisticated priors to 
variables in the uncollapsed model.

We first show results on a Gaussian mixture model (GMM), written
$x_i \sim \mathcal{N}(\mu_{z_i},\Sigma_{z_i}), \hspace{1mm}
z_i \sim \text{Dir}(\alpha), \hspace{1mm}
\{\mu_m\}_{m=1}^M \sim \pi, \hspace{1mm}
i=1\text{,...,}n$. Using a normal $\pi_f$ over
$\{\mu_m\}_{m=1}^M$ allows for
collapsed Gibbs sampling.
We also show results on a topic model (latent
Dirichlet allocation (LDA) \cite{blei2003latent})
for text data (for the form of this model,
see \cite{blei2003latent,wang2011collaborative}). 
Here, using a Dirichlet
$\pi_f$ over topics allows for
collapsed Gibbs sampling.
For mixture models, we generate synthetic data
from the above model ($n$=10,000, $d$=2, $M$=9), 
and for topic models, we use the Simple English 
Wikipedia\footnote{\url{https://simple.wikipedia.org/}} 
corpus (n=27,443 documents, vocab=10,192 words), 
and set $M$=400 topics.

In Fig.~\ref{fig:latentFactor}, we show results for mixture
and topic models. In (a) we show inferred posteriors over 
GMM components for a number of relational target priors, which we 
define in (b). In (c), we apply the diversity-promoting target
prior to LDA, to separate redundant topics. Here, we show
two topic clusters (``geography'' and ``family'') in 
$p_f(\theta|x^n)$, which are separated into distinct, 
yet thematically-similar, topics 
after prior swapping. In (a) and (c)
we also show wall times of the inference methods.
%%%%%%%%%%%%%%%%%%%%%%%%%%%%%%%%%%%%%%%%%%%%%%%%%%
%%%%          END OF EMPIRICAL RESULTS 
%%%%%%%%%%%%%%%%%%%%%%%%%%%%%%%%%%%%%%%%%%%%%%%%%%

%%%%%%%%%%%%%%%%%%%%%%%%%%%%%%%%%%%%%%%%%%%%%%%%%%
%%%%          CONCLUSION AND BIB 
%%%%%%%%%%%%%%%%%%%%%%%%%%%%%%%%%%%%%%%%%%%%%%%%%%

\vspace{-2mm}
\section{Conclusion}
\vspace{-1mm}
Given some false posterior inference
result, and an arbitrary target prior,
we have studied methods to accurately compute the
associated target posterior (or expectations with
respect to it), and to do this efficiently
by leveraging the pre-inferred result.
We have argued and shown empirically
that this strategy is effective even
when the false and target posteriors
are quite dissimilar.
We believe that this strategy
shows promise to allow a wider range of 
(and possibly less-costly) inference 
alorithms to be applied to certain models,
and to allow updated or new prior 
information to be more-easily incorporated
into models without re-incurring the 
full costs of standard inference algorithms.

\section{Acknowledgements}
W. Neiswanger and E. Xing are supported by 
NSF BigData Grant IIS1447676.

%%%%            BIBLIOGRAPHY
\bibliography{main.bib}
\bibliographystyle{icml2017}

%%%%%%%%%%%%%%%%%%%%%%%%%%%%%%%%%%%%%%%%%%%%%%%%%%
%%%%         END OF CONCLUSION AND BIB 
%%%%%%%%%%%%%%%%%%%%%%%%%%%%%%%%%%%%%%%%%%%%%%%%%%

%%%%%%%%%%%%%%%%%%%%%%%%%%%%%%%%%%%%%%%%%%%%%%%%%%
%%%%                APPENDIX 
%%%%%%%%%%%%%%%%%%%%%%%%%%%%%%%%%%%%%%%%%%%%%%%%%%

\newpage

\addtolength{\oddsidemargin}{.25in}
\addtolength{\evensidemargin}{.25in}
\addtolength{\textwidth}{-0.5in}

\appendix

\onecolumn
\begin{center}
{\Large \textbf{Appendix for ``Post-Inference Prior Swapping''}}
\end{center}
\vspace{2mm}
\normalsize

\section{Details on the IS Example (Sec.~\ref{importanceSampling})}
Here we provide details on the IS example (for a 
normal $\pi_f$ and Laplace $\pi$) given in 
Sec.~\ref{importanceSampling}.

We made the following
statement: 
if $p_f(\theta|x^n) = \mathcal{N}(\theta|m,s^2)$,
in order for 
$|\mu_h - \mathbb{E}_{p_f}[\hat{\mu}_h^\text{IS}]| < \delta$,
we need
\begin{align*}
    T \geq \exp \left\{ \frac{1}{2s^2} 
    (|\mu_h - m| - \delta)^2 \right\}.
\end{align*}

To show this, we first give an upper bound on the 
expected value of the maximum of $T$ zero-mean $s^2$-variance
Gaussian random variables.
Let $\{\tilde{\theta}_t\}_{t=1}^T \sim g$, where
$g(\theta) = \mathcal{N}(\theta|0,s^2)$,
and let $Z = \max_t \{\tilde{\theta}_t\}_{t=1}^T$. Then, for some $b>0$,
\begin{align*}
    \exp\{b \mathbb{E}_g[Z]\} \leq \mathbb{E}_g[ \exp\{bZ\}]
    = \mathbb{E}_g\left[\max_t \left\{ \exp\{b \tilde{\theta}_t\} 
    \right\}_{t=1}^T \right]
    \leq \sum_{t=1}^T \mathbb{E}_g \left[\exp\{b \tilde{\theta}_t\}\right]
    = T \exp\{b^2 s^2 / 2\},
\end{align*}
where the first inequality is due to Jensen's inequality, 
and the final equality is due to the definition of
a Gaussian moment generating function.
The above implies that
\begin{align*}
    \mathbb{E}_g[Z] \leq \frac{\log T}{b} + \frac{b s^2}{2}.
\end{align*}

Setting $b = \sqrt{\frac{2}{s^2} \log T}$, we have that
\begin{align*}
    \mathbb{E}_g\left[
    \max_t \{\tilde{\theta}_t\}_{t=1}^T \right]
    = \mathbb{E}_g[Z]
    \leq s \sqrt{2 \log T}.
\end{align*}

However, note that for all $\{\tilde{\theta}_t\}_{t=1}^T$,
and weights $\{w(\tilde{\theta}_t)\}_{t=1}^T$ (such that
$\sum_{t=1}^T w(\tilde{\theta}_t) = 1$),
the IS estimate $\hat{\mu}_h^\text{IS}$ for $h(\theta) = \theta$ 
must be less than or equal to $\max_t \{\tilde{\theta}_t\}_{t=1}^T$
(since the weighted average of $\{\tilde{\theta}_t\}_{t=1}^T$ 
cannot be larger than the maximum of this set). Therefore, 

\begin{align*}
    \mathbb{E}_g\left[\hat{\mu}_h^\text{IS}\right] 
    \leq \mathbb{E}_g\left[\max_t \{\tilde{\theta}_t\}_{t=1}^T\right]
    \leq s \sqrt{2 \log T},
\end{align*}
and equivalently
\begin{align*}
    T \geq
    \exp\left\{ 
    \frac{1}{2s^2} \mathbb{E}_g\left[\hat{\mu}_h^\text{IS}\right]^2
    \right\}.
\end{align*}

In our example,  we wanted the expected estimate to be
within $\delta$ of $\mu_h$, i.e. we wanted 
$|\mu_h - \mathbb{E}_g[\hat{\mu}_h^\text{IS}]| < \delta$
$\iff$
% $-\delta - \mu_h   \leq
% - \mathbb{E}_{p_f}[\hat{\mu}_h^\text{IS}] \leq
% \delta - \mu_h$.
% $\iff$
$\delta - \mu_h \leq 
\mathbb{E}_g[\hat{\mu}_h^\text{IS}] \leq
\mu_h + \delta$,
and therefore,
\begin{align*}
    T \geq
    \exp\left\{ 
    \frac{1}{2s^2} \mathbb{E}_g\left[\hat{\mu}_h^\text{IS}\right]^2
    \right\}
    \geq
    \exp\left\{ 
    \frac{1}{2s^2} \left(\delta -\mu_h\right)^2
    \right\}.
\end{align*}

Finally, notice that the original statement involved samples
$\{\tilde{\theta}_t\}_{t=1}^T \sim 
p_f(\theta|x^n) = \mathcal{N}(m,s^2)$ 
(instead of from $g = \mathcal{N}(0,s^2)$). But this is 
equivalent to setting $p_f(\theta|x^n) = g(\theta)$,
and shifting our goal so that we want
$\delta - |\mu_h-m| \leq 
\mathbb{E}_{p_f}[\hat{\mu}_h^\text{IS}] \leq
|\mu_h-m| + \delta$. This gives us the desired bound:

\begin{align*}
    T \geq
    \exp\left\{ 
    \frac{1}{2s^2} \mathbb{E}_{p_f}\left[\hat{\mu}_h^\text{IS}\right]^2
    \right\}
    \geq
    \exp\left\{ 
    \frac{1}{2s^2} \left(\delta -|\mu_h-m|\right)^2
    \right\}.
\end{align*}

\section{Prior Swapping Pseudocode (for a false posterior PDF inference result $\tilde{p}_f(\theta)$)}

Here we give pseudocode for 
the prior swapping procedure,
given some false posterior PDF
inference result $\tilde{p}_f(\theta)$,
using the prior swap functions
$p_s(\theta) \propto \frac{\tilde{p}_f(\theta) \pi(\theta)}{\pi_f(\theta)}$ 
and $\nabla_\theta \log p_s(\theta) \propto 
\nabla_\theta \log \tilde{p}_f(\theta) + \nabla_\theta \log \pi(\theta) 
- \nabla_\theta \log \pi_f(\theta)$,
as described in Sec.~\ref{priorswapping}.

In Alg.~\ref{alg:PS_MH}, we show prior swapping via the Metropolis-Hastings 
algorithm, which makes repeated use of $p_s(\theta)$. 
In Alg.~\ref{alg:PS_HMC} we show prior swapping via Hamiltonian
Monte Carlo, which makes repeated use of $\nabla_\theta \log p_s(\theta)$.
A special case of Alg.~\ref{alg:PS_HMC}, which occurs
when we set the number of simulation steps to $L=1$ (in line 6),
is prior swapping via Langevin dynamics.

\begin{algorithm}[!ht]
    \caption{Prior swapping via Metropolis-Hastings.} 
    \label{alg:PS_MH}
    \KwIn{Prior swap function $p_s(\theta)$, and proposal $q$.}
    \vspace{2pt}
    \KwOut{Samples $\{\theta_t\}_{t=1}^T \sim p_s(\theta)$ 
    as $T\rightarrow \infty$.}
    \vspace{2pt}
    Initialize $\theta_0$.
            \hspace{27mm}$\triangleright$ Initialize Markov chain.\\
    \For{$t = 1,\ldots,T$}{
        Draw $\theta_s \sim q(\theta_s \mid \theta_{t-1})$.
            \hspace{5mm}$\triangleright$ Propose new sample.\\
        Draw $u \sim \text{Unif}([0,1])$.\\
        \If{$u < \min \left\{ 1, 
        \frac{p_s(\theta_s)q(\theta_t \mid \theta_s)}
        {p_s(\theta_t)q(\theta_s \mid \theta_t)}
        \right\}$}{
            Set $\theta_t \leftarrow \theta_s$.
            \hspace{17mm}$\triangleright$ Accept proposed sample.\\
        }
        \Else{
            Set $\theta_t \leftarrow \theta_{t-1}$.
            \hspace{14mm}$\triangleright$ Reject proposed sample.\\
        }
    }
\end{algorithm}

\begin{algorithm}[!ht]
    \caption{Prior swapping via Hamiltonian Monte Carlo.} 
    \label{alg:PS_HMC}
    \KwIn{Prior swap function $p_s(\theta)$, its 
    gradient-log $\nabla_\theta \log p_s(\theta)$, and 
    step-size $\epsilon$.}
    \vspace{2pt}
    \KwOut{Samples $\{\theta_t\}_{t=1}^T \sim p_s(\theta)$ 
    as $T\rightarrow \infty$.}
    \vspace{2pt}
    Initialize $\theta_0$.
            \hspace{41mm}$\triangleright$ Initialize Markov chain.\\
    \For{$t = 1,\ldots,T$}{
        Draw $r_t \sim \mathcal{N}(0,I)$.\\
        Set $(\widetilde{\theta}_0,\widetilde{r}_0) 
        \leftarrow (\theta_{t-1},r_{t-1})$ \\
        Set $\widetilde{r}_0 \leftarrow
        \widetilde{r}_0 + \frac{\epsilon}{2} 
        \nabla_\theta \log p_s(\widetilde{\theta}_0)$ .
            \hspace{7mm}$\triangleright$ Propose new sample (next 4 lines).\\
        \For{$l = 1,\ldots,L$}{
        Set $\widetilde{\theta}_l \leftarrow
        \widetilde{\theta}_{l-1} + \epsilon \widetilde{r}_{l-1}$.\\ 
        Set $\widetilde{r}_l \leftarrow
        \widetilde{r}_{l-1} + \epsilon 
        \nabla_\theta \log p_s(\widetilde{\theta}_l)$.\\
        }
        Set $\widetilde{r}_L \leftarrow
        \widetilde{r}_L + \frac{\epsilon}{2} 
        \nabla_\theta \log p_s(\widetilde{\theta}_L)$ .\\
        Draw $u \sim \text{Unif}([0,1])$.\\
        \If{$u < \min \left\{ 1, 
        \frac{p_s(\widetilde{\theta}_L)  
        \widetilde{r}_L^\top \widetilde{r}_L}{
        p_s(\theta_{t-1}) r_{t-1}^\top r_{t-1}}
        \right\}$}{
            Set $\theta_t \leftarrow \widehat{\theta}_L$.
            \hspace{30mm}$\triangleright$ Accept proposed sample.\\
        }
        \Else{
            Set $\theta_t \leftarrow \theta_{t-1}$.
            \hspace{27mm}$\triangleright$ Reject proposed sample.\\
        }
    }
\end{algorithm}

\newpage
\section{Proofs of Theoretical Guarantees}

Here, we prove the theorems
stated in Sec.~\ref{sampleBasedPS}.

Throughout this analysis, we assume that we have $T$ 
samples $\{\tilde{\theta}_{t} \}_{t=1}^{T_f}$ 
$\subset$ $\mathcal{X}$ $\subset$ $\mathbb{R}^d$ 
from the false-posterior $p_f(\theta|x^n)$, 
and that $b \in \mathbb{R}_+$ denotes the
bandwidth of our semiparametric false-posterior density 
estimator $\tilde{p}_f^{sp}(\theta)$. 
Let H\"{o}lder class $\Sigma(2,L)$ on $\mathcal{X}$ be defined as the set of
all $\ell = \lfloor 2 \rfloor$ times differentiable functions
$f:\mathcal{X} \rightarrow \mathbb{R}$ whose derivative $f^{(l)}$ satisfies
\begin{equation*}
    \lvert f^{(\ell)}(\theta) - f^{(\ell)}(\theta') \rvert
    \leq L \left\lvert \theta-\theta' \right\rvert^{2-\ell} \hspace{3mm}
    \text{ for all } \hspace{1mm} \theta,\theta' \in \mathcal{X}.
\end{equation*}

Let the class of densities $\mathcal{P}(2,L)$ be
\begin{equation*}
    \mathcal{P}(2,L) = \left\{  f \in \Sigma(2,L) 
    \hspace{1mm}\Big|\hspace{1mm} f \geq 0,
    \int f(\theta) d\theta=1 \right\}.
\end{equation*}

Let data $x^n = \{x_1,\ldots,x_n\} \subset \mathcal{Y} 
\subset \mathbb{R}^p$, 
let $\mathcal{Z} \subset \mathcal{Y}$ be any 
set such that $x^n \subset \mathcal{Z}$, and let
$\mathcal{F}_\mathcal{Z}(L)$ denote the set of densities 
$p:\mathcal{Y} \rightarrow \mathbb{R}$ that satisfy
\begin{equation*}
    |\log p(x) - \log p(x')|
    \leq
    L|x - x'|, 
    \hspace{3mm} \text{ for all } \hspace{1mm}
    x,x' \in \mathcal{Z}.
\end{equation*}

In the following theorems, we assume that the false-posterior density 
$p_f(\theta|x^n)$ is bounded, i.e. that
there exists some $B>0$ such that $p_f(\theta|x^n) \leq B$ for all
$\theta \in \mathbb{R}^d$; 
that the prior swap density $p_s(\theta) \in \mathcal{P}(2,L)$; 
and that the model family $p(x^n|\theta) \in 
\mathcal{F}_\mathcal{Z}(L)$ for some $\mathcal{Z}$.

\vspace{5mm}
\begin{namedtheorem}[2.1]
For any $\alpha = (\alpha_1,\ldots,\alpha_k) \subset \mathbb{R}^p$
and $k>0$ let 
$\tilde{p}_f^\alpha(\theta)$ be defined 
as in Eq.~(\ref{eq:paraFalsePost}).
Then, there exists $M>0$ such that 
$\frac{p_f(\theta|x^n)}{\tilde{p}_f^\alpha(\theta)} < M$, for all $\theta \in \mathbb{R}^d$.
\end{namedtheorem}
\begin{proof}
To prove that there exists $M>0$ such that 
$\frac{p_f(\theta|x^n)}{\tilde{p}_f^\alpha(\theta)}<M$, 
note that the false posterior can be written
\begin{align*}
    p_f(\theta|x^n) = \frac{1}{Z_1} \pi_f(\theta) \prod_{i=1}^n L(\theta|x_i)
    = \frac{1}{Z_1} \pi_f(\theta) \prod_{i=1}^n p(x_i|\theta),
\end{align*}
and the parametric estimate $\tilde{p}_f^\alpha(\theta)$
is defined to be
\begin{align*}
 \tilde{p}_f^\alpha(\theta)
    = \frac{1}{Z_2}
    \pi_f(\theta)
    \prod_{j=1}^k p(\alpha_j|\theta)^{n/k}.
\end{align*}

Let $d = \max_{i,j} |x_i - \alpha_j|$. For any
$i \in \{1,\ldots,n\}$, $j \in \{1,\ldots,k\}$, 
\begin{align*}
    |\log p(x_i|\theta) - \log p(\alpha_j|\theta) |
    \leq Ld
    \implies
    \left|\log \frac{p(x_i|\theta)}{p(\alpha_j|\theta)}\right|
    \leq Ld,
\end{align*}
and
\begin{align*}
    \exp\left\{\log \frac{p(x_i|\theta)}{p(\alpha_j|\theta)}\right\}
    \leq 
    \exp\left\{  
    \left|\log \frac{p(x_i|\theta)}{p(\alpha_j|\theta)}\right|
    \right\}
    \leq \exp\{Ld\}
    \implies
    \frac{p(x_i|\theta)}{p(\alpha_j|\theta)} \leq \exp\{Ld\}.
\end{align*}

Therefore
\begin{align*}
    \frac{p_f(\theta|x^n)}{\tilde{p}_f^\alpha(\theta)}
    \leq 
    \frac{Z_2}{Z_1}
    \frac{\prod_{i=1}^n p(x_i|\theta)}
    {\prod_{j=1}^k p(\alpha_j|\theta)^{n/k}}
    \leq
    \frac{Z_2}{Z_1}\exp\{nLd\} = M.
\end{align*}

\end{proof}

\vspace{5mm}
\begin{namedcorollary}[2.1.1]
For $\{\theta_t\}_{t=1}^T \sim p_s^\alpha(\theta) \propto 
\frac{\tilde{p}_f^\alpha(\theta) \pi(\theta)}{\pi_f(\theta)}$,
$w(\theta_t) = 
\frac{p_f(\theta_t|x^n)}{\tilde{p}_f^\alpha(\theta_t)}
\left(\sum_{r=1}^T
\frac{p_f(\theta_r|x^n)}{\tilde{p}_f^\alpha(\theta_r)}
\right)^{-1}$,
and test function that satisfies 
$\text{Var}_p\left[h(\theta)\right] < \infty$,
the variance of IS estimate $\hat{\mu}_h^\text{PSis} = 
\sum_{t=1}^T h(\theta_t) w(\theta_t)$
is finite.
\end{namedcorollary}
\begin{proof}
This follows directly from the sufficient conditions
for finite variance IS estimates
given by \cite{geweke1989bayesian},
which we have proved are satisfied
for $\hat{\mu}_h^\text{PSis}$ in Theorem 2.1.
\end{proof}

% \vspace{5mm}
% \begin{namedtheorem}[2.2]
% Given false posterior samples 
% $\{\tilde{\theta}_t\}_{t=1}^T \sim f_{\phi|x}$
% and $h \asymp T^{-1/(4+d)}$, the estimator 
% $\hat{f}_{\widetilde{ps}}^s$ (Eq.~\ref{eq:semiparEst}) is 
% consistent, i.e. its mean-squared error satisfies
% \begin{align}
%     \sup_{ f_{ps} \in \mathcal{P}(2,L) } \hspace{2mm}
%     \mathbb{E}\left[ \int\left(
%             \hat{f}_{\widetilde{ps}}^s(\theta) - 
%             f_{\widetilde{ps}}(\theta)\right)^2 
%             d\theta \right]
%     < \frac{c}{T^{4/(4+d)}} \nonumber
% \end{align}
% for some $c>0$ and $0<h\leq 1$.
% \end{namedtheorem}

\vspace{5mm}
\begin{namedtheorem}[2.2]
Given false posterior samples 
$\{\tilde{\theta}_t\}_{t=1}^{T_f} \sim p_f(\theta|x^n)$
and $b \asymp T_f^{-1/(4+d)}$, the estimator
$p_s^{sp}$ is 
consistent for $p(\theta|x^n)$, 
i.e. its mean-squared error satisfies
\begin{align}
    \sup_{ p(\theta|x^n) \in \mathcal{P}(2,L) } \hspace{2mm}
    \mathbb{E}\left[ \int\left(
            p_s^{sp}(\theta) - 
            p(\theta|x^n)\right)^2 
            d\theta \right]
    < \frac{c}{T_f^{4/(4+d)}} \nonumber
\end{align}
for some $c>0$ and $0<b\leq 1$.
\end{namedtheorem}

\begin{proof}
To prove mean-square consistency 
of our semiparametric prior swap 
density estimator $p_s^{sp}$,
we give a bound on the 
mean-squared error (MSE), and show that it 
tends to zero as we increase the
number of samples $T_f$ drawn from the false-posterior. 
To prove this, we bound
the bias and variance of the estimator,
and use this to bound the MSE.
In the following, to avoid cluttering notation,
we will drop the subscript $p_f$ in $\mathbb{E}_{p_f}[\cdot]$.

We first bound the bias of our semiparametric
prior swap estimator. 
For any $p(\theta|x^n) \in \mathcal{P}(2,L)$, we can write the bias as 
\begin{align*}
        \left| \mathbb{E} \left[ p_s^{sp}(\theta) \right] 
        - p(\theta|x^n) \right| 
        &=
        c_1 \left| \mathbb{E} \left[ \tilde{p}_f^{sp}(\theta) 
        \frac{\pi(\theta)}{\pi_f(\theta)} \right] - 
        p_f(\theta|x^n) 
        \frac{\pi(\theta)}{\pi_f(\theta)} \right|  \\
        &=
        c_2 \left| \frac{\pi(\theta)}{\pi_f(\theta)} 
        \mathbb{E} \left[ \tilde{p}_f^{sp}(\theta) \right] - 
        p_f(\theta|x^n)
        \right|  \\
        &= 
        c_3 \left| \mathbb{E} \left[ \tilde{p}_f^{sp}(\theta) 
        \right] - p_f(\theta|x^n) \right| \\
        &\leq ch^2
\end{align*}
for some $c>0$, where we have 
used the fact that $ \left| \mathbb{E}
\left[ \tilde{p}_f^{sp}(\theta) \right] - p_f(\theta|x^n) \right|
\leq  \tilde{c} h^2$ for some~$\tilde{c} > 0$ 
(given in \cite{hjort1995nonparametric,wasserman2006all}).

We next bound the variance of our 
semiparametric prior swap estimator.
For any $p(\theta|x^n) \in \mathcal{P}(2,L)$, we can write the
variance of our estimator as 
\begin{align*}
    \text{Var} \left[ p_s^{sp}(\theta) \right]
    &= c_1 \text{Var} \left[ \tilde{p}_f^{sp}(\theta) 
        \frac{\pi(\theta)}{\pi_f(\theta)} \right] \\
    &= \frac{\pi(\theta)^2}{\pi_f(\theta)^2} \text{Var}
    \left[ \tilde{p}_f^{sp}(\theta) \right] \\
    &\leq \frac{c}{T_f h^d}
\end{align*}
for some $c>0$, where we have used the facts that 
$\text{Var}\left[ \tilde{p}_f^{sp}(\theta) \right]
    \leq \frac{c}{Th^d}$ for some $c>0$ and 
    $\mathbb{E}\left[ \tilde{p}_f^{sp}(\theta) \right]^2 
    \leq \tilde{c}$ for some $\tilde{c}>0$
(given in \cite{hjort1995nonparametric,wasserman2006all}).
Next, we will use these two results to bound the 
mean-squared error of our semiparametric
prior swap estimator, which shows that it is 
mean-square consistent.

We can write the mean-squared error as the sum of the variance and
the bias-squared, and therefore,
\begin{align*}
    \mathbb{E}\left[ \int\left(p_s^{sp}(\theta) - 
        p(\theta|x^n)\right)^2 d\theta \right] 
        &\leq c_1 h^2 + \frac{c_2}{Th^d} \\
        &= \frac{c}{T_f^{4/(4+d)}}
\end{align*}
for some $c>0$, using the fact that $h \asymp T_f^{-1/(4+d)}$.
\end{proof}

\vspace{3mm}

\section{Further Empirical Results}

Here we show further empirical results on
a logistic regression model
with hierarchical target prior given by 
$ \pi = \mathcal{N}(0,\alpha^{-1}I), \hspace{1mm}
\alpha \sim \text{Gamma}(\gamma,1)$.
We use synthetic data so that
we are able to compare the timing and 
posterior error of different
methods as we tune $n$ and $d$.

In this experiment, we assume that we are given
samples from a false posterior $p_f(\theta|x^n)$,
and we want to most-efficiently compute
the target posterior under prior $\pi(\theta)$.
In addition to the prior swapping methods,
we can run standard iterative inference algorithms, such
as MCMC or variational inference (VI), on the 
target posterior (initializing them, 
for example, at the false posterior mode)
as comparisons.
The following experiments aim to show that,
once the data size $n$ grows large enough,
prior swapping methods become more
efficient than standard inference algorithms.
They also aim to show that the held-out test error
of prior swapping matches that of 
these standard inference algorithms.
In these experiments, we also add a prior 
swap method called \emph{prior swapping VI};
this method involves making a VI approximation
to $p_f(\theta|x^n)$, and using it for 
$\tilde{p}_f(\theta)$. Prior swapping VI
allows us to see whether the test error is 
similar to standard VI inference algorithms,
which compute some approximation to the posterior.
Finally, we show results over a range of 
target prior hyperparameter 
values $\gamma$ to show that prior swapping
maintains accuracy (i.e. has a similar error
as standard inference algorithms) over
the full range.

We show results in Fig.~\ref{fig:lrAll}. 
In (a) and (b) we vary the number of 
observations ($n$=10-120,000) and see that prior 
swapping has a constant wall time while
the wall times of both MCMC and VI increase with $n$. 
In (b) we see that the prior 
swapping methods achieve the same test error as the standard
inference methods. In (c) and (d) we vary the number
of dimensions ($d$=1-40). In this case, all methods have 
increasing wall time, and again the test errors match.
In (e), (f), and (g), we vary the prior hyperparameter 
($\gamma$=1-1.05). For prior swapping,
we infer a single $\tilde{p}_f(\theta)$ (using $\gamma = 1.025$)
with both MCMC and VI applied to $p_f(\theta|x^n)$, 
and compute \emph{all other} 
hyperparameter results using this $\tilde{p}_f(\theta)$.
This demonstrates that prior swapping can quickly infer
correct results over a range of hyperparameters.
Here, the prior swapping semiparametric method matches 
the test error of MCMC slightly better than the 
parametric method.

\begin{figure*}[!ht]
        \makebox[\textwidth][c]{
        \hspace{-1mm}
        \includegraphics[width=1\textwidth]{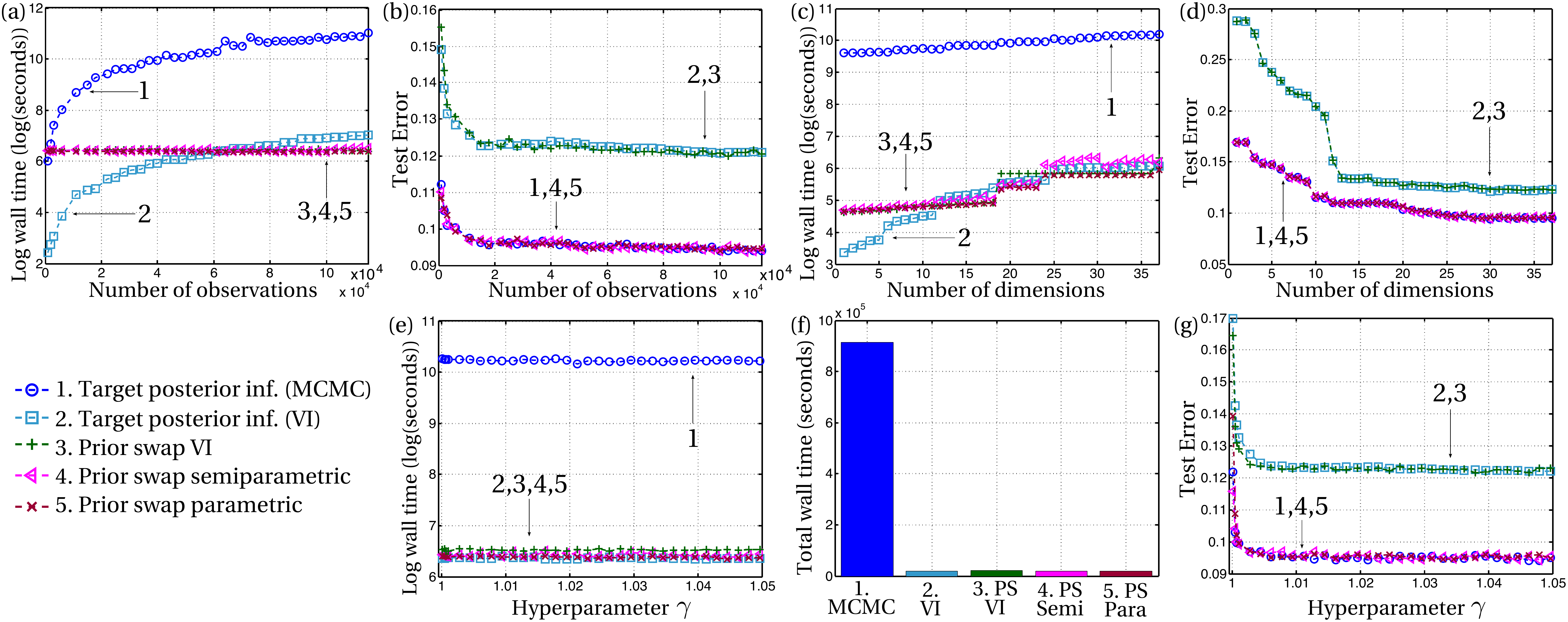}}
        \vspace{-1mm} \vspace*{-1mm}
        \caption{\label{fig:lrAll} Bayesian hierarchical 
        logistic regression: (a-b) Wall time and test error
        comparisons for varying data size $n$. As $n$ is 
        increased, wall time remains constant for prior 
        swapping but grows for standard inference methods.
        (c-d) Wall time and test error comparisons for 
        varying model dimensionality $d$. (e-g) 
        Wall time and test error comparisons for 
        inferences on a set of prior hyperparameters
        $\gamma \in [1,1.05]$.
        Here, a single false posterior 
        $\tilde{p}_f(\theta)$ (computed at $\gamma = 1.025$) 
        is used for prior swapping
        on all other hyperparameters.
        }
\end{figure*}

\end{document}